\relax
\documentclass[twoside,11pt]{article}

\usepackage{jmlr2e}

\usepackage{times}  %Required
\usepackage{helvet}  %Required
\usepackage{courier}  %Required
\usepackage{url}  %Required
\usepackage{graphicx}  %Required
\usepackage{booktabs}

\usepackage{xcolor} 
\usepackage{amssymb,amsmath}
\usepackage{dsfont}
\usepackage{textpos}
\usepackage{tikz}
\usetikzlibrary{shapes}
\usepackage{array}
\newcolumntype{P}[1]{>{\centering\arraybackslash}p{#1}}

\newcommand{\floor}[1]{\lfloor #1 \rfloor}

\newcommand{\eps}{\varepsilon}
\newcommand{\RL}{\mathbb{R}}
\newcommand{\norm}[1]{\|#1\|}
\newcommand{\rk}[1]{\mathrm{rank}(#1)}
\newcommand{\Poi}[1]{\mathrm{Poi}(#1)}
\newcommand{\Ex}[1]{\mathbb{E}\left[#1\right]}
\newcommand{\Var}[1]{\mathbb{V}\left[#1\right]}
\newcommand{\Lik}[1]{\mathcal{L}\left(#1\right)}
\newcommand{\amin}[1]{\mathrm{argmin}_{#1}}

\DeclareMathOperator{\pa}{\mathbf{pa}}

\tikzstyle{pred}=[draw,ellipse,minimum size=7mm,inner sep=0,thick]
\tikzstyle{varnode}=[draw,circle,minimum size=7mm,inner sep=0mm]%
\tikzstyle{observed}=[fill=black!20]%
\tikzstyle{leaf}=[draw,ellipse,inner sep=.2em]%
\tikzstyle{innernode}=[draw,rectangle,inner sep=.5em]%
\firstpageno{1}

\begin{document}

\title{Coresets for Dependency Networks}
\author{coresets, dependency networks, graphical models, structure discovery, Gaussian, Poisson, big data}

\author{\name Alejandro Molina
    \email first.last@tu-dortmund.de \\
    \addr CS Department \\
    TU Dortmund, Germany
    \AND
    \name Alexander Munteanu
    \email first.last@tu-dortmund.de \\
    \addr CS Department \\
    TU Dortmund, Germany
    \AND
    \name Kristian Kersting
    \email last@cs.tu-darmstadt.de \\
    \addr CS Dept.~and Centre for CogSci \\
    TU Darmstadt, Germany
}

\editor{}
\maketitle

\begin{abstract}
Many applications infer the structure of a probabilistic graphical model from data
to elucidate the relationships between variables. But how can we train graphical models on a massive data set? 
In this paper, we show how to construct coresets---compressed data sets which can be used as proxy for the original data and have provably bounded worst case error---for Gaussian dependency networks (DNs), i.e., 
cyclic directed graphical models over Gaussians, 
where the parents of each variable are its Markov blanket. Specifically, we prove that Gaussian DNs admit coresets of size independent of the size of the data set. Unfortunately, this does not extend to DNs over members of the exponential family in general. As we will prove, Poisson DNs do not admit small coresets. Despite this worst-case result, we will provide
an argument why our coreset construction for DNs can still work well in practice on count data.
To corroborate our theoretical results, we empirically evaluated the resulting Core DNs on 
real data sets. The results demonstrate significant gains over no or naive sub-sampling,
even in the case of count data. 
\end{abstract}

\section{Introduction}

\noindent Artificial intelligence and machine learning  have  achieved  considerable  successes  in recent  years,  and  an  ever-growing  number  of  disciplines rely on them. Data is now ubiquitous, and there is great value from  understanding  the  data,  building  e.g. probabilistic graphical models to elucidate the relationships between variables. In the big data era, however, scalability has become crucial for any useful machine learning approach. In this paper, we consider the problem of training graphical models, in particular Dependency Networks~\cite{heckerman00}, on massive data sets. They are cyclic directed graphical models, where the parents of each variable are its Markov blanket, and have been proven successful in various tasks, such as
collaborative filtering~\cite{heckerman00}, 
phylogenetic analysis~\cite{carlsonBRBMKMWHGH08}, genetic analysis~\cite{dobra09,phatakKCW10}, network inference from sequencing data~\cite{Allen2013},
and traffic as well as topic modeling~\cite{hadiji2015mlj_pdns}.

Specifically, we show that Dependency Networks over Gaussians---arguably one of the most prominent type of distribution in statistical machine learning---admit coresets of size independent of the size of the data set. Coresets are weighted subsets of the data, which guarantee that models fitting them will also provide a good fit for the original data set, and  %So far, coresets 
have been studied before for clustering \cite{BadoiuHI02,feldmanFK11,FeldmanSS13,lucicBK16}, classification %\cite{Har-PeledRZ07,Har-Peled15,ReddiPS15},
\cite{Har-PeledRZ07,Har-Peled15,ReddiPS15}, 
regression \cite{DrineasMM06,DrineasMM08,DasguptaDHKM09,geppert2017random},
%\cite{DrineasMM08,DasguptaDHKM09,geppert2017random}, 
and the smallest enclosing ball problem 
%\cite{BadoiuC08,MunteanuSF14,AgarwalS15};
\cite{BadoiuC03,BadoiuC08,MunteanuSF14,AgarwalS15}; 
we refer to \cite{Phillips17} for a recent extensive literature overview. Our contribution continues this line of research and generalizes the use of coresets to probabilistic graphical modeling.

Unfortunately, this coreset result does not extend to Dependency Networks over members of the exponential family in general. We prove that Dependency Networks over Poisson random variables~\cite{Allen2013,hadiji2015mlj_pdns} do not admit (sublinear size) coresets: every single input point is important for the model and needs to appear in the coreset.
This is an important negative result, since count data---the primary target of Poisson distributions---is at the center of many scientific 
endeavors from citation counts to web page hit counts, from counts of
procedures in medicine to the count of births and deaths in
census, from counts of words in a document to the count
of gamma rays in physics. 
Here, modeling one event such as
the number of times a certain lab test yields a particular
result can provide an idea of the number of potentially invasive
procedures that need to be performed on a patient.
Thus, elucidating the relationships between variables can yield great insights
into %the properties of these 
massive count data.
Therefore, despite our worst-case result, we will provide
an argument why our coreset construction for Dependency Networks can still work well in practice on count data.
To corroborate our theoretical results, we empirically evaluated the resulting Core Dependency Networks (CDNs)  on 
several real data sets. The results demonstrate significant gains over no or naive sub-sampling,
even for count data.

We proceed as follows. We review Dependency Networks (DNs), prove that Gaussian DNs admit sublinear size coresets, and discuss the possibility to generalize this result to count data. Before concluding, we illustrate our theoretical results empirically.

\section{Dependency Networks}
%explain dependency networks 
%explain how they can be implemented by glms and what are glms
%explain how to do inference

Most of the existing AI and machine learning literature on graphical models
%---we refer to~\cite{koller09pgmbook} for a general introduction---
is dedicated to binary, multinominal, or certain classes of continuous (e.g. Gaussian) random variables.
%Typically one distinguishes between directed and undirected models.
Undirected models, aka~{\it Markov Random Fields} (MRFs), such as Ising (binary random variables) and Potts (multinomial random variables) models have found a lot of applications in various fields
such as robotics, computer vision and statistical physics, among others.
%In one of the most simple forms, there are e.g. Ising grids, which are defined over binary random variables and follow a regular lattice structure.
%These models have been proven successfully in computer vision and are used e.g.~to described phenomena of interacting atoms in statistical physics, e.g. ferromagnetism.
%Generalizing the Ising model to the non-binary, but finite state space with lattice structure, there are Potts models.
%Besides computer vision, Potts models have been used in numerous other scientific disciplines. However, in many applications we require models with unrestricted structures and infinite ranges.
%Although undirected models allow arbitrary structures, finding an optimal structure is a hard problem and even learning structures approximately is already computational challenging.
Whereas MRFs allow for cycles in the structures, %In contrast to MRFs,
directed models aka~{\it Bayesian Networks} (BNs) required acyclic directed relationships among the random variables. 
%They have also been used in a number 
%of applications such as planning, NLP and information retrieval, among others.

{\it Dependency Networks} (DNs)---the focus of the present paper---combine concepts from directed and undirected worlds and are due to \citeauthor{heckerman00}~(\citeyear{heckerman00}). Specifically,
like BNs, DNs have directed arcs but they allow for networks with cycles and bi-directional arcs, akin to MRFs. This makes
DNs quite appealing for many applications because we can build multivariate models from univariate distributions~\cite{Allen2013,yang15,hadiji2015mlj_pdns}, %\cite{YangRAL12,Allen2013,hadiji2015mlj_pdns}
while still permitting  efficient structure learning using local estimtatiors or
gradient tree boosting.
Generally, if the data are fully observed, learning is done locally on the level of the conditional probability distributions for each variable mixing
directed and indirected as needed. 
Based on these local distributions, samples from the joint distribution are obtained via Gibbs sampling.
Indeed, the Gibbs sampling neglects the question of a consistent joint probability distribution and instead makes only use of local distributions.
The generated samples, however, are often sufficient to answer many probability queries.
%However, since the local distributions might be parametrized separately,  there is no guarantee
%that there exists a joint distribution of which they are the conditionals of.
%Moreover, \citeauthor{bengioLAY14}~(\citeyear{bengioLAY14}) have recently proven the existence of a consistent distribution per given evidence, which does not have to be known in closed form,
%as long as an unordered Gibbs sampler converges.
%In any case, except for few cases that we will discuss below, graphical models have been mainly investigated for non-count distributions.

%Consider a Gaussian dependency network (GDN), i.e., a collection of Gaussian linear regression models
%\[
%\Psi = \{p_i(X^{(i)}|X^{\setminus i},\gamma^{(i)}) \sim \mathcal{N}(X^{\setminus i}\gamma^{(i)},\sigma^2)\;|\; i\in[d]\}.
%\]
%Our derivations extend naturally to the setting where given a digraph structure $\mathcal G =([d],\mathcal E), \mathcal E \subseteq [d]\times [d] \setminus \{(i,i)\,|\,i\in [d] \} $, $X^{(i)}$ is regressed on $X^{(\mathbf{pa}_i)}$, where $\mathbf{pa}_i = \{j \,|\, (j,i)\in\mathcal E \}$ \cite{heckerman00}.

%The logarithm of the \emph{(pseudo-)likelihood} \cite{Besag75} of the above model is given by
%\[
%\ln \Lik{ \Psi } = \ln \prod p_i = \sum \ln p_i.
%\]

%DNs are by nature paralellizable at every feature
%reference pdn and other dns, also the one from nature
Formally, let $X=({X}^{(1)}, \ldots, {X}^{(d)})$ denote a random vector and ${x}$ its instantiation.
%Sets of random variables are written as $\bf X$ and correspondingly their instantiations as $\bf x$.
%Given a set of random variables $\mathbf{X} = (\mathrm{X}^{(1)}, \ldots, \mathrm{X}^{(d)})$ distributed according to $\Psi = (p_1, \ldots, p_d)$,
A \emph{Dependency Network} (DN) on $X$ is a pair $(\mathcal G, \Psi)$ where $\mathcal G = (\mathcal V, \mathcal E)$ is a directed, possibly cyclic, graph where each node in $\mathcal V=[d]=\{1,\ldots,d\}$ corresponds to the random variable $X^{(i)}$.
%Hence, we can use nodes in $G$ and the random variables in $\bf X$ interchangeably.
In the set of directed edges $\mathcal E \subseteq \mathcal V\times \mathcal V \setminus \{(i,i)\,|\,i\in [d] \}$, each edge models a dependency between variables, i.e., if there is no edge between $i$ and $j$ then the variables $X^{(i)}$ and $X^{(j)}$ are conditionally independent given the other variables $X^{\setminus i,j}$ indexed by $[d]\setminus \{i,j\}$ in the network.
%Here, ${\mathbf X}^{\setminus i,j}$ is shorthand for $\mathbf{X} \setminus \{\mathrm{X}^{(i)}, \mathrm{X}^{(j)}\}$.
We refer to the nodes that have an edge pointing to ${X}^{(i)}$ as its parents, denoted by $\pa_i = \{{X}^{(j)} \,|\, (j,i)\in\mathcal E \}$.
$\Psi=\{p_i\,|\, i\in[d]\}$ is a set of conditional probability distributions associated with each variable $X^{(i)}\sim p_i$, where
\begin{equation*}
p_i = p({x}^{(i)} |\, {\pa_i}) = p({x}^{(i)} |\, {{x}^{\setminus i}})\;.
\end{equation*}
As example of such a local model, consider Poisson conditional probability distributions as illustrated in Fig.~\ref{fig:pdn_ex} (left):
\begin{equation*}
p({x}^{(i)} |\, {\pa_i}) = \frac{\lambda_i(x^{\setminus i})^{{x}^{(i)}}}{{x}^{(i)}!} e^{-\lambda_i(x^{\setminus i})}\;.
\end{equation*}
Here, $\lambda_i( x^{\setminus i})$ highlights the fact that the mean can have a functional form that is dependent on ${X}^{(i)}$'s parents.
Often, we will refer to it simply as $\lambda_i$. 
The construction of the local conditional probability distribution is similar to the (multinomial) Bayesian network case. However, in the case of DNs, the graph is not necessarily acyclic and $p({x}^{(i)} |\, {x}^{\setminus i})$ typically has an infinite range, and hence cannot be represented using a finite table of probability values.
Finally, the full joint distribution is simply defined as the product of local distributions:
\begin{equation*}
p({\bf x}) = \prod\nolimits_{i\in[d]}p({x}^{(i)} |\, {x}^{\setminus i})\;,
\end{equation*}
also called pseudo likelihood. For the Poisson case, this reads
\begin{equation*}
p({\bf x}) =  \prod\nolimits_{i\in[d]} \frac{\lambda_i^{{x}^{(i)}}}{{x}^{(i)}!} e^{-\lambda_i}\;.
\end{equation*}
Note, however, that doing so does not guarantee the existence of a consistent joint distribution, i.e., a joint distribution of which they are the conditionals. \citeauthor{bengioLAY14}~(\citeyear{bengioLAY14}), however, have recently proven the existence of a consistent distribution per given evidence, which does not have to be known in closed form,
as long as an unordered Gibbs sampler converges.

\begin{figure}[t]
    \centering
        \includegraphics[width=0.45\columnwidth]{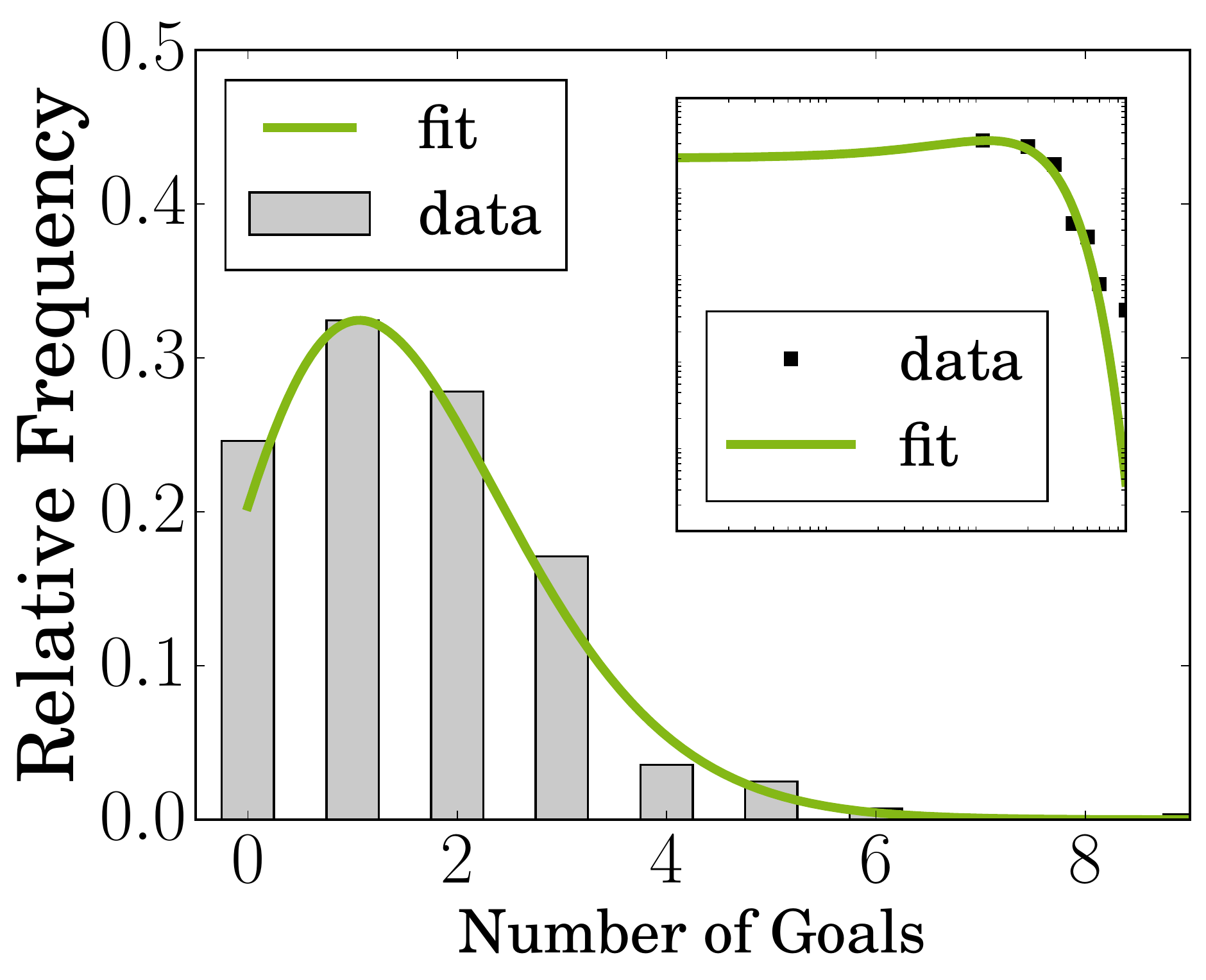}
    \quad\quad%
	    \begin{tikzpicture}[node distance=1.4cm]
% \def \n {3}
% \def \radius {1cm}
% \def \margin {22} % margin in angles, depends on the radius

% \foreach \s in {1,...,\n}
% {
% \node[draw, circle] (\s) at ({360/\n * (\s - 1)}:\radius) {$X_\s$};
% }
% \path (1) edge[bend right=15,-latex] (2);
% \path (2) edge[bend right=15,-latex] (1);
% \path (1) edge[bend right=15,-latex] (3);
% \path (3) edge[bend right=15,-latex] (1);
% \path (2) edge[bend right=15,-latex] (3);
% \path (3) edge[bend right=15,-latex] (2);

\node (z) at (1.4,-3.5) {$~$}; %only for positioning
\node[varnode,yshift=-1cm,minimum size=.90cm] (x0) at (0,0) {$\,X^{(0)}$};
\node[varnode, below of=x0,xshift=-1cm,minimum size=.90cm] (x1) {$\,X^{(1)}$};
\node[varnode, below of=x0,xshift=1cm,minimum size=.90cm] (x2) {$\,X^{(2)}$};
\draw[-latex] (x0) -- (x1);
\draw[-latex] (x0) -- (x2);
\draw[->,-latex] (x1) to [bend right=15] (x2);
\draw[->,-latex] (x2) to [bend right=15] (x1);

%\node[varnode,observed] (x0) at (0,0) {$X_{0}$};
%\node[varnode, below of=x0,xshift=-1cm] (x1) {$X_{1}$};
%\node[varnode, below of=x0,xshift=1cm] (x2) {$X_{2}$};
%\draw[-latex] (x0) -- (x1);
%\draw[-latex] (x0) -- (x2);
%\draw[->,-latex] (x1) to [bend right=15] (x2);
%\draw[->,-latex] (x2) to [bend right=15] (x1);

\end{tikzpicture}
    \caption{
    Illustration of Dependency Networks (DNs) using Poissons. {\bf (left)}
    The number of goals scored in soccer games follows a Poisson distribution. The plot shows the distribution of home goals in the season 2012/13 of the German Bundesliga by the home team. The home team scored on average $\lambda = 1.59$ goals per game.
    {\bf (right)} Example structure of a Poisson DN. The conditional distribution of each count variable given its neighbors is a Poisson distribution.
    Similar to a Bayesian network a Poisson DN is directed, however, it also contains cycles.
    (Best viewed in color)%
    \label{fig:pdn_ex}}
\end{figure}
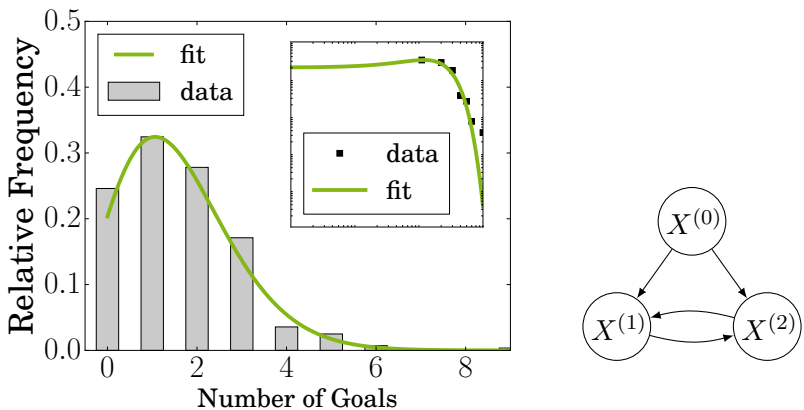

\section{Core Dependency Networks}
%explain what are coresets and how can they be used to learn glms
As argued, learning Dependency Networks (DNs) amounts to determining the conditional probability distributions from a given set of $n$ training instances $x_i\in\RL^d$ representing the rows of the data matrix $X\in\RL^{n\times d}$ over $d$ variables.
Assuming that $p({x}^{(i)} |\, {\pa_i})$ is parametrized as a generalized linear model (GLM)~\cite{mccullagh89glms}, this amounts to estimating the parameters $\gamma^{(i)}$ of the GLM associated with each variable ${X}^{(i)}$, since this completely determines the local distributions, % are completely determined by the GLM parameters $\gamma^{(i)}$.
but $p({x}^{(i)} |\, {\pa_i})$ will possibly depend on all other variables in the network, and these dependencies define the structure of the network.
This view of training DNs as fitting $d$ GLMs to the data allows us to develop {\it Core Dependency Networks} (CDNs): Sample a coreset and train a DN over certain members of the GLM family on the sampled corest.
%---as we will show now. 

%\subsection{Coresets}
A \emph{coreset} is a (possibly) weighted and usually considerably smaller subset of the input data that approximates a given objective function for all candidate solutions: 
\begin{definition}[$\varepsilon$-coreset]
	Let $X$ be a set of points from a universe $U$ and let $\Gamma$ be a set of candidate solutions. Let $f:U\times \Gamma\rightarrow \mathbb{R}^{\geq 0}$ be a non-negative measurable function. Then a set $C \subset X$ is an $\varepsilon$-coreset of $X$ for $f$, if 
	\[\forall \gamma\in \Gamma: |f(X,\gamma)-f(C,\gamma)|\leq \varepsilon\cdot f(X,\gamma).\]
\end{definition}
We now introduce the formal framework that we need towards the design of coresets for learning dependency networks. 
A very useful structural property for $\ell_2$ based objective (or loss) functions is the concept of an $\eps$-subspace embedding.

\begin{definition}[$\eps$-subspace embedding]
	\label{def:epsSE}
	An $\eps$-subspace embedding for the columnspace of $X$ is a matrix $S$ such that
	%\begin{eqnarray*}
	\[
		\forall \gamma \in {\RL}^d: (1-\eps)\norm{X\gamma}^2\leq\norm{SX\gamma}^2\leq(1+\eps)\norm{X\gamma}^2
	\]%\end{eqnarray*}
\end{definition}
\noindent We can construct a sampling matrix $S$ which forms an $\eps$-subspace embedding with constant probabilty in the following way:
Let $U$ be any orthonormal basis for the columnspace of $X$. This basis can be obtained from the singular value decomposition (SVD) $X=U\Sigma V^T$ of the data matrix. Now let $\rho=\rk{U}=\rk{X}$ and define the \emph{leverage scores} $l_i = \norm{U_{i*}}^2/\norm{U}_F^2=\norm{U_{i*}}^2/\rho$ for $i\in[n]$. Now we fix a sampling size parameter $k=O(\rho\log(\rho/\eps) / \eps^2)$, sample the input points one-by-one with probability $q_i = \min\{1, k\cdot l_i\}$ and reweight their contribution to the loss function by $w_i = 1/q_i$. Note that, for the sum of squares loss, this corresponds to defining a diagonal (sampling) matrix $S$ by $S_{ii} = 1/\sqrt{q_i}$ with probability $q_i$ and $S_{ii} = 0$ otherwise. Also note, that the expected number of samples is $k=O(\rho\log(\rho/\eps) / \eps^2)$, which also holds with constant probability by Markov's inequality. Moreover, to give an intuition why this works, note that for any fixed $\gamma\in\RL^d$, we have
\[
\Ex{ \norm{SX\gamma}^2 } = \sum \left(\frac{x_i\gamma}{\sqrt{q_i}}\right)^2 q_i = \sum (x_i\gamma)^2 = \norm{X\gamma}^2.
\]
The significantly stronger property of forming an $\eps$-subspace embedding, according to Definition \ref{def:epsSE}, follows from a matrix approximation bound given in \cite{RudelsonV07,DrineasMM08}.%, which is originally from \cite{RudelsonV07}.
\begin{lemma}
	Let $X$ be an input matrix with $\rk{X}=\rho$. Let $S$ be a sampling matrix constructed as stated above with sampling size parameter $k=O(\rho\log(\rho/\eps) / \eps^2)$. Then $S$ forms an $\eps$-subspace embedding for the columnspace of $X$ with constant probability.
\end{lemma}

\begin{proof} Let $X=U\Sigma V^T$ be the SVD of $X$. By Theorem 7 in \cite{DrineasMM08} there exists an absolute constant $C>1$ such that 
	\begin{eqnarray*}
		\Ex{ \norm{U^TS^TSU - U^TU} } &\leq& C\sqrt{\frac{\log k}{k}} \norm{U}_F \norm{U} \\
		&\leq& C\sqrt{\frac{\log k}{k}} \sqrt{\rho} \;\leq\; \eps,
	\end{eqnarray*}
	where we used the fact that $\norm{U}_F = \sqrt\rho$ and $\norm{U}=1$ by orthonormality of $U$. The last inequality holds by choice of $k=D\rho\log(\rho/\eps) / \eps^2$ for a large enough absolute constant $D>1$ such that $\frac{1+\log D}{ D } < \frac{1}{4C^2}$, since
	\begin{align*}
		&\frac{\log k}{k} = \frac{\log(D \rho\log(\rho/\eps) / \eps^2 )}{ D\rho\log(\rho/\eps) / \eps^2 }
		\leq  \frac{2 \eps^2 \log(D\rho \log(\rho/\eps)/\eps )}{ D\rho\log(\rho/\eps) } \\
		&\leq  \frac{4 \eps^2(\log(\rho/\eps) + \log D )}{ D\rho\log(\rho/\eps)} 
		\leq  \frac{4\eps^2}{\rho} \left( \frac{1+\log D}{ D } \right) < \frac{\eps^2}{C^2\rho}\;. 
	\end{align*}
	%\begin{align*}
	%	&\frac{\log k}{k} = \frac{\log(D \rho\log(\rho/\eps) / \eps^2 )}{ D\rho\log(\rho/\eps) / \eps^2 } 
	%	\leq  \frac{2 \eps^2 \log(D(\rho/\eps) \log(\rho/\eps) )}{ D\rho\log(\rho/\eps) } \\
	%	&\leq  \frac{4 \eps^2(\log(\rho/\eps) + \log D )}{ D\rho\log(\rho/\eps)} 
	%	\leq  \frac{4\eps^2}{\rho} \left( \frac{1+\log D}{ D } \right) \;<\; \frac{\eps^2}{C^2\rho}\;. 
	%\end{align*}
	By an application of Markov's inequality and rescaling $\eps$, we can assume with constant probability 
	\begin{equation}
		\label{eqn:epsineq}
		\norm{U^TS^TSU - U^TU} \leq \eps.
	\end{equation}
	We show that this implies the $\eps$-subspace embedding property. To this end, fix $\gamma\in\RL^d.$
	\begin{eqnarray*}
		&~& |\,\norm{SX\gamma}^2 - \norm{X\gamma}^2\,| \\
		&=& \norm{\gamma^TX^TS^TSX\gamma - \gamma^TX^TX\gamma}\\
		&=& \norm{\gamma^TV\Sigma U^TS^TSU\Sigma V^T\gamma - \gamma^TV\Sigma U^TU\Sigma V^T\gamma} \\
		&=& \norm{\gamma^TV\Sigma\, (U^TS^TSU - U^TU)\, \Sigma V^T\gamma} \\
		&\leq& \norm{U^TS^TSU - U^TU} \cdot \norm{\Sigma V^T\gamma}^2 \\
		&\leq& \norm{U^TS^TSU - U^TU} \cdot \norm{X\gamma}^2 \;\leq\; \eps \norm{X\gamma}^2,
	\end{eqnarray*}
	The first inequality follows by submultiplicativity, and the second from rotational invariance of the spectral norm. Finally we conclude the proof by Inequality (\ref{eqn:epsineq}).
\end{proof}

The question arises whether we can do better than $O(\rho\log(\rho/\eps) / \eps^2)$. One can show by reduction from the coupon collectors theorem that there is a lower bound of $\Omega(\rho \log \rho)$ matching the upper bound up to its dependency on $\varepsilon$. The hard instance is a $d^m \times d, m\in\mathbb{N}$ orthonormal matrix in which the scaled canonical basis $\mathds{I}_d/\sqrt{d^{m-1}}$ is stacked $d^{m-1}$ times. The leverage scores are all equal to $1/d^m$, implying a uniform sampling distribution with probability $1/d$ for each basis vector. Any rank $\rho=d$ preserving sample must comprise at least one of them. This is exactly the coupon collectors theorem with $d$ coupons which has a lower bound of $\Omega(d\log d)$ \cite{MotwaniR95}. The fact that the sampling is without replacement does not change this, since the reduction holds for arbitrary large $m$ creating sufficient multiple copies of each element to simulate the sampling with replacement \cite{Tropp11}.

%\subsection{Coresets for Gaussian Dependency Networks}
Now we know that with constant probability over the randomness of the construction algorithm, $S$ satisfies the $\eps$-subspace embedding property for a given input matrix $X$. This is the structural key property to show that actually $SX$ is a coreset for Gaussian linear regression models and dependency networks. Consider $(\mathcal G, \Psi)$, a Gaussian dependency network (GDN), i.e., a collection of Gaussian linear regression models
\[
\Psi = \{p_i(X^{(i)}|X^{\setminus i},\gamma^{(i)}) = \mathcal{N}(X^{\setminus i}\gamma^{(i)},\sigma^2)\;|\; i\in[d]\}
\]
on an arbitrary digraph structure $\mathcal G$ \cite{heckerman00}.
%Our derivations extend naturally to arbitrary digraph structure $\mathcal G =(\mathbf{X},\mathcal E), \mathcal E \subseteq \mathbf{X}\times \mathbf{X} \setminus \{(\mathrm{X}_i,\mathrm{X}_i)\,|\,i\in [d] \} $, $X^{(i)}$ is regressed on $X^{(\mathbf{pa}_i)}$, where $\mathbf{pa}_i = \{j \,|\, (j,i)\in\mathcal E \}$ \cite{heckerman00}.
The logarithm of the \emph{(pseudo-)likelihood} \cite{Besag75} of the above model is given by
\[
\ln \Lik{ \Psi } = \ln \prod p_i = \sum \ln p_i.
\]
A maximum likelihood estimate can be obtained by maximizing this function with respect to $\gamma = (\gamma^{(1)},\ldots,\gamma^{(d)})$ which is equivalent to minimizing the GDN loss function
\[
f_{G}(X,\gamma) = \sum \norm{X^{\setminus i}\gamma^{(i)}-X^{(i)}}^2.
\]

\begin{theorem}
	\label{thm:epsSE->epsCS}
	Given $S$, an $\eps$-subspace embedding for the columnspace of $X$ as constructed above, $SX$ is an $\eps$-coreset of $X$ for the GDN loss function.% That is
%	\begin{eqnarray*}
%		\forall \gamma = (\gamma^{(1)},\ldots,\gamma^{(d)})\in\RL^{d(d-1)}: \\
%		(1-\eps) \sum \norm{X^{\setminus i}\gamma^{(i)}-X^{(i)}}^2 \\
%		\leq \sum \norm{SX^{\setminus i}\gamma^{(i)}-SX^{(i)}}^2 \\
%		\leq (1+\eps) \sum \norm{X^{\setminus i}\gamma^{(i)}-X^{(i)}}^2
%	\end{eqnarray*}
\end{theorem}
\begin{proof}
	Fix an arbitrary $\gamma = (\gamma^{(1)},\ldots,\gamma^{(d)})\in\RL^{d(d-1)}$. Consider the affine map $\Phi: \RL^{d-1}\times [d] \rightarrow \RL^d$, defined by $\Phi(\gamma^{(i)}) =  \mathds{I}_d^{\setminus i}\gamma^{(i)} - e_i$. Clearly $\Phi$ extends its argument from $d-1$ to $d$ dimensions by inserting a $-1$ entry at position $i$ and leaving the other entries in their original order. Let $\beta^{(i)} = \Phi(\gamma^{(i)})\in\RL^d$. Note that for each $i\in [d]$ we have
	\begin{equation}
		\label{eqn:identitiy}
		X\beta^{(i)} = X\Phi(\gamma^{(i)}) = X^{\setminus i}\gamma^{(i)}-X^{(i)},
	\end{equation}
	and each $\beta^{(i)}$ is a vector in $\RL^d$. Thus, the triangle inequality and the universal quantifier in Definition \ref{def:epsSE} guarantee that
	\begin{eqnarray*}
		&~& | \sum \norm{SX\beta^{(i)}}^2 - \sum \norm{X\beta^{(i)}}^2 \,| \\
		&=& | \sum ( \norm{SX\beta^{(i)}}^2 - \norm{X\beta^{(i)}}^2 ) \,| \\
		&\leq& \sum | \norm{SX\beta^{(i)}}^2 - \norm{X\beta^{(i)}}^2 \,| \\
		&\leq& \sum \eps \norm{X\beta^{(i)}}^2 \;=\; \eps \sum \norm{X\beta^{(i)}}^2 .
	\end{eqnarray*}
	The claim follows by substituting Identity (\ref{eqn:identitiy}).
\end{proof}
It is noteworthy that computing one single coreset for the columnspace of $X$ is sufficient, rather than computing $d$ coresets for the $d$ different subspaces spanned by $X^{\setminus i}$.

From Theorem \ref{thm:epsSE->epsCS} it is straightforward to show that the minimizer found for the coreset is a good approximation of the minimizer for the original data.
\begin{corollary}
	Given an $\eps$-coreset $C$ of $X$ for the GDN loss function, let $\tilde \gamma \in \amin{\gamma\in\RL^{d(d-1)}}f_{G}(C,\gamma)$. Then it holds that
	\[
		f_{G}(X,\tilde \gamma) \leq (1+4\eps) \min_{\gamma\in\RL^{d(d-1)}} f_{G}(X,\gamma).
	\]
\end{corollary}
\begin{proof}
	Let $\gamma^* \in \amin{\gamma\in\RL^{d(d-1)}}f_{G}(X,\gamma)$. Then
	\begin{eqnarray*}
		f_{G}(X,\tilde \gamma) &\leq & \frac{1}{1-\eps} f_{G}(C,\tilde \gamma) \;\;\,\leq \;\, \frac{1}{1-\eps} f_{G}(C,\gamma^*) \\
		&\leq & \frac{1+\eps}{1-\eps} f_{G}(X,\gamma^*) \;\leq \;\, {(1+4\eps)} f_{G}(X,\gamma^*).
	\end{eqnarray*}
	The first and third inequalities are direct applications of the coreset property, the second holds by optimality of $\tilde \gamma$ for the coreset, and the last follows from $\eps<\frac{1}{2}.$
\end{proof}
\noindent Moreover, the coreset does not affect inference within GDNs. Recently, it was shown for (Bayesian) Gaussian linear regression models that the entire multivariate normal distribution over the parameter space is approximately preserved by $\eps$-subspace embeddings \cite{geppert2017random}, which generalizes the above. This implies that the coreset yields a useful pointwise approximation in Markov Chain Monte Carlo inference via random walks like the pseudo-Gibbs sampler in \cite{heckerman00}.

\section{Negative Result on Coresets for Poisson DNs}
Naturally, the following question arises: Do  (sublinear size) coresets exist for dependency networks over the exponential family in general? Unfortunately, the answer is no! Indeed, there is no (sublinear size) coreset for the simpler problem of Poisson regression, which implies the result for Poisson DNs. We show this formally by reduction from the communication complexity problem known as \emph{indexing}. 

To this end, recall that the negative log-likelihood for Poisson regression is \cite{mccullagh89glms,Winkelmann08} \[\ell(\gamma):=\ell(\gamma| X,Y) = \sum \exp(x_i \gamma)-y_i\cdot x_i\gamma + \ln(y_i!).\]
\begin{theorem}
	\label{thm:poissonLB}
	Let $\Sigma_D$ be a data structure for $D=[X,Y]$ that approximates likelihood queries $\Sigma_D(\gamma)$ for Poisson regression, such that \[\forall \gamma\in\RL^d: \eta^{-1}\cdot\ell(\gamma|D) \leq \Sigma_D(\gamma) \leq \eta\cdot\ell(\gamma|D).\] If $\eta<{\frac{\exp(\frac{n}{4})}{2n^2}}$ then $\Sigma_D$ requires $\Omega(n)$ bits of storage.
\end{theorem}
\begin{proof}
	We reduce from the indexing problem which is known to have $\Omega(n)$ one-way randomized communication complexity \cite{JayramKS08}. Alice is given a vector $b\in\{0,1\}^n$. She produces for every $i$ with $b_i=1$ the points $x_i=(r\cdot\omega^i,-1)\in\RL^{3}$, where $\omega^i, i\in\{0,\ldots,n-1\}$ denote the $n^{th}$ unit roots in the plane, i.e., the vertices of a regular $n$-polygon of radius $r=n/(1-\cos( \frac{2\pi}{n} )) \leq n^3$ in canonical order. The corresponding counts are set to $y_i=1$. She builds and sends $\Sigma_D$ of size $s(n)$ to Bob, whose task is to guess the bit $b_j$. He chooses to query $\gamma = (\omega^j,r \cdot \cos( \frac{2\pi}{n} ) )\in\RL^{3}$. Note that this affine hyperplane separates $r\cdot\omega^j$ from the other scaled unit roots since it passes exactly through $r\cdot\omega^{(j-1) \bmod n}$ and $r\cdot\omega^{(j+1) \bmod n}$. Also, all points are within distance $2r$ from each other by construction and consequently from the hyperplane. Thus, $-2r \leq x_i\gamma \leq 0$ for all $i\neq j$.% while. $2r \geq x_j\gamma > 0$.	
	
	If $b_j=0$, then $x_j$ does not exist and the cost is at most
	\begin{eqnarray*}
		\ell(\gamma) &=& \sum \exp(x_i\gamma)-y_i\cdot x_i\gamma + \ln(y_i!) \\
		&\leq & \sum 1 +2r + 1
		\leq  2n + 2nr \leq 4n^4\;.
	\end{eqnarray*}
	
	If $b_j=1$ then $x_j$ is in the expensive halfspace and at distance exactly 
	\begin{eqnarray*}
		x_j\gamma &=& (r\omega^j)^T\omega^j - r \cdot \cos\left( \frac{2\pi}{n} \right)\\
		&=& r \cdot \left(1-\cos\left( \frac{2\pi}{n} \right)\right) \;=\; n
	\end{eqnarray*}
	So the cost is bounded below by	$\ell(\gamma) \geq \exp(n) - n + 1 \geq \exp(\frac{n}{2})$.
	
	Given $\eta < \frac{\exp(\frac{n}{4})}{2n^2}$, Bob can distinguish these two cases based on the data structure only, by deciding whether $\Sigma_D(\gamma)$ is strictly smaller or larger than $\exp(\frac{n}{4})\cdot 2n^2$. Consequently $s(n)=\Omega(n)$, since this solves the indexing problem.
\end{proof}

Note that the bound is given in bit complexity, but restricting the data structure to a sampling based coreset and assuming every data point can be expressed in $O(d\log n)$ bits, this means we still have a lower bound of $k=\Omega(\frac{n}{\log n})$ samples. 
\begin{corollary}
	Every sampling based coreset for Poisson regression with approximation factor $\eta<{\frac{\exp(\frac{n}{4})}{2n^2}}$ as in Theorem \ref{thm:poissonLB} requires at least $k=\Omega(\frac{n}{\log n})$ samples.
\end{corollary}
At this point it seems very likely that a similar argument can be used to rule out any $o(n)$-space constant approximation algorithm. This remains an open problem for now.

\section{Why Core DNs for Count Data can still work}
So far, we have a quite pessimistic view on extending CDNs beyond Gaussians. In the Gaussian setting, where the loss is measured in squared Euclidean distance, the number of important points, i.e., having significantly large leverage scores, is bounded essentially by $O(d)$. This is implicit in the original early works \cite{DrineasMM08} and has been explicitly formalized later \cite{LangbergS10,ClarksonW13}. It is crucial to understand that this is an inherent property of the norm function, and thus holds for arbitrary data. For the Poisson GLM, in contrast, we have shown that its loss function does not come with such properties from scratch. We constructed a worst case scenario, where basically every single input point is important for the model and needs to appear in the coreset.
Usually, this is not the case with statistical models, where the data is assumed to be generated i.i.d. from some generating distribution that fits the model assumptions. Consider for instance a data reduction for Gaussian linear regression via leverage score sampling vs. uniform sampling. It was shown that given the data follows the model assumptions of a Gaussian distribution, the two approaches behave very similarly. Or, to put it another way, the leverage scores are quite uniform. In the presence of more and more outliers generated by the heavier tails of $t$-distributions, the leverage scores increasingly outperform uniform sampling \cite{MaMY15}.

The Poisson model 
\begin{eqnarray}
	\label{eqn:PoissonModel}
	y_i &\sim &  \Poi{\lambda_i}, \;\lambda_i = \exp(x_i\gamma).
\end{eqnarray}
though being the standard model for count data, suffers from its inherent limitation on equidispersed data since $\Ex{y_i|x_i} = \Var{y_i|x_i} = \exp(x_i\gamma)$. Count data, however, is often overdispersed especially for large counts. This is due to unobserved variables or problem specific heterogeneity and contagion-effects. The log-normal Poisson model is known to be inferior for data which specifically follows the Poisson model, but turns out to be more powerful in modeling the effects that can not be captured by the simple Poisson model. It has wide applications for instance in econometric elasticity problems. We review the log-normal Poisson model for count data \cite{Winkelmann08}
\begin{eqnarray*}
y_i &\sim & \Poi{\lambda_i},\\
\lambda_i &=& \exp(x_i\gamma)u_i = \exp(x_i\gamma+v_i),\\
v_i \;=\; \ln u_i &\sim& \mathcal{N}\left(\mu,\sigma\right).
\end{eqnarray*}
A natural choice for the parameters of the log-normal distribution is $\mu = -\frac{\sigma^2}{2}$ in which case we have \begin{eqnarray*}
	\Ex{y_i|x_i} &=& \exp( x_i\gamma + \mu + {\sigma^2}/{2}) = \exp( x_i\gamma)\, ,\\
	\Var{y_i|x_i} &=& \Ex{y_i|x_i} +(\exp(\sigma^2)-1)\Ex{y_i|x_i}^2.
\end{eqnarray*}
It follows that $\Var{y_i|x_i} =\exp( x_i\gamma) + \Omega(\exp(x_i\gamma)^2) > \exp( x_i\gamma)$, where a constant $\sigma^2$ that is independent of $x_i$, controls the amount of overdispersion. Taking the limit for $\sigma \rightarrow 0$ we arrive at the simple model (\ref{eqn:PoissonModel}), since the distribution of $v_i=\ln u_i$ tends to $\delta_0$, the deterministic Dirac delta distribution which puts all mass on $0$. The inference might aim for the log-normal Poisson model directly as in \cite{ZhouLDC12}, or it can be performed by (pseudo-)maximum likelihood estimation of the simple Poisson model. The latter provides a consistent estimator as long as the log-linear mean function is correctly specified, even if higher moments do not possess the limitations inherent in the simple Poisson model \cite{Winkelmann08}. 

Summing up our review on the count modeling perspective, we learn that preserving the log-linear mean function in a Poisson model is crucial towards consistency of the estimator. Moreover, modeling counts in a log-normal model gives us intuition why leverage score sampling can capture the underlying linear model accurately: In the log-normal Poisson model, $u$ follows a log-normal distribution. It thus holds for $\ln \lambda = X\gamma + \ln u \,=\, X\gamma + v,$ that% where $v\sim \mathcal{N}(-\frac{\sigma^2}{2}\cdot\mathds{1},\sigma^2 \mathds{I}_n)$
\begin{eqnarray*}
	%\ln \lambda &=& X\gamma + \ln u \,=\, X\gamma + v, \\
	v&\sim &\mathcal{N}\left(-\frac{\sigma^2}{2}\cdot\mathds{1},\sigma^2 \mathds{I}_n\right)
\end{eqnarray*}
by independence of the observations, which implies %$\ln \lambda \sim \mathcal{N}(X\gamma-\frac{\sigma^2}{2}\cdot\mathds{1},\sigma^2 \mathds{I}_n).$
\begin{eqnarray*}
\ln \lambda &\sim &\mathcal{N}\left(X\gamma-\frac{\sigma^2}{2}\cdot\mathds{1},\sigma^2 \mathds{I}_n\right). \end{eqnarray*}
Omitting the bias $\mu=-\frac{\sigma^2}{2}$ in each intercept term (which can be cast into $X$), we notice that this yields again an ordinary least squares problem $\norm{X\gamma-\ln(\lambda)}^2$
%\begin{eqnarray*}
%\label{eqn:OLS}
%\norm{X\gamma-\ln(\lambda)}^2 %\min_{\gamma\in\RL^d} 
%\end{eqnarray*}
defined in the columspace of $X$.

\begin{figure*}[t!]
\centering
\begin{tabular}{P{0.31\linewidth}|P{0.31\linewidth}|P{0.31\linewidth}}
\includegraphics[width=1\linewidth]{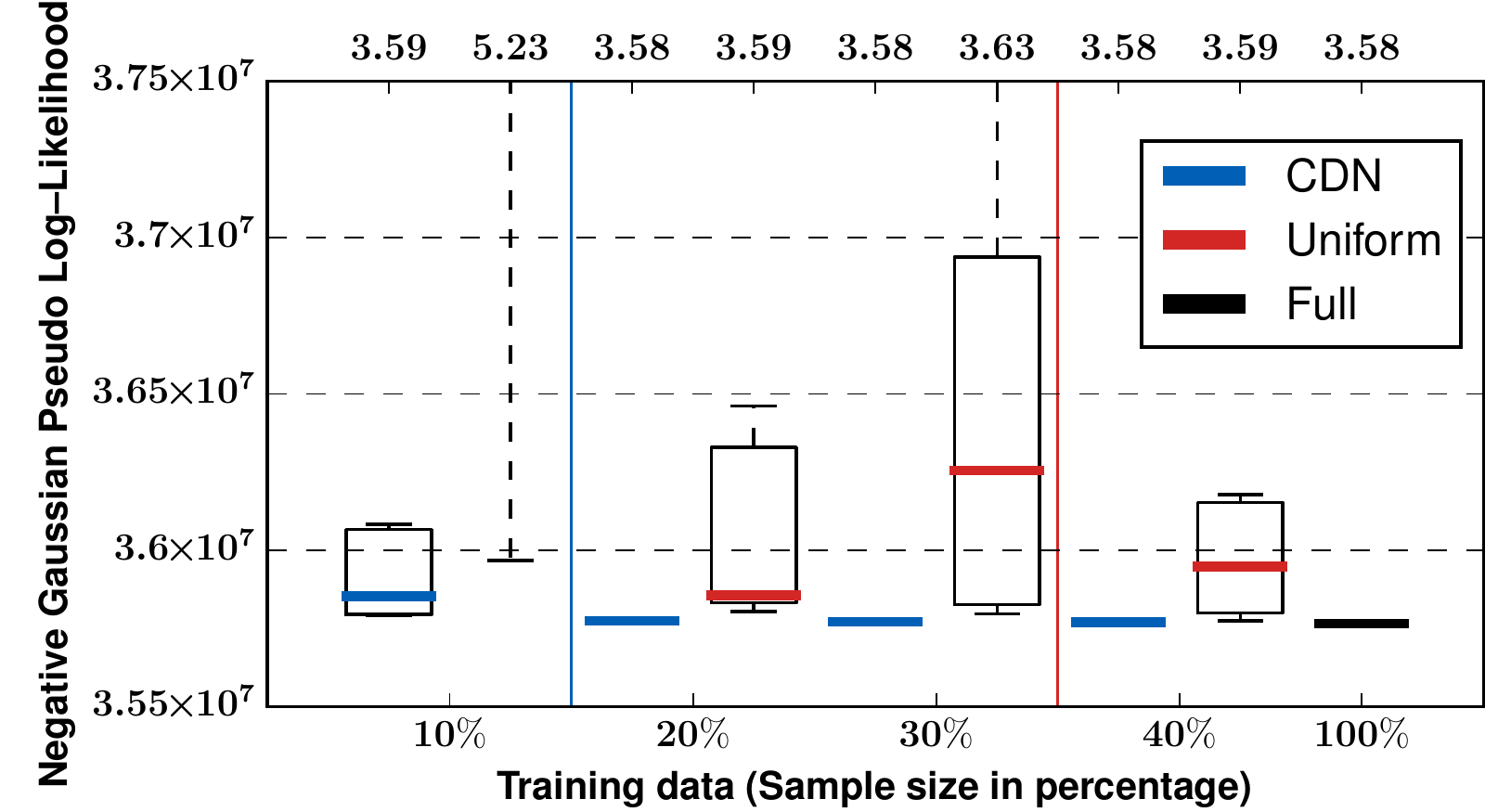}
&\includegraphics[width=1\linewidth]{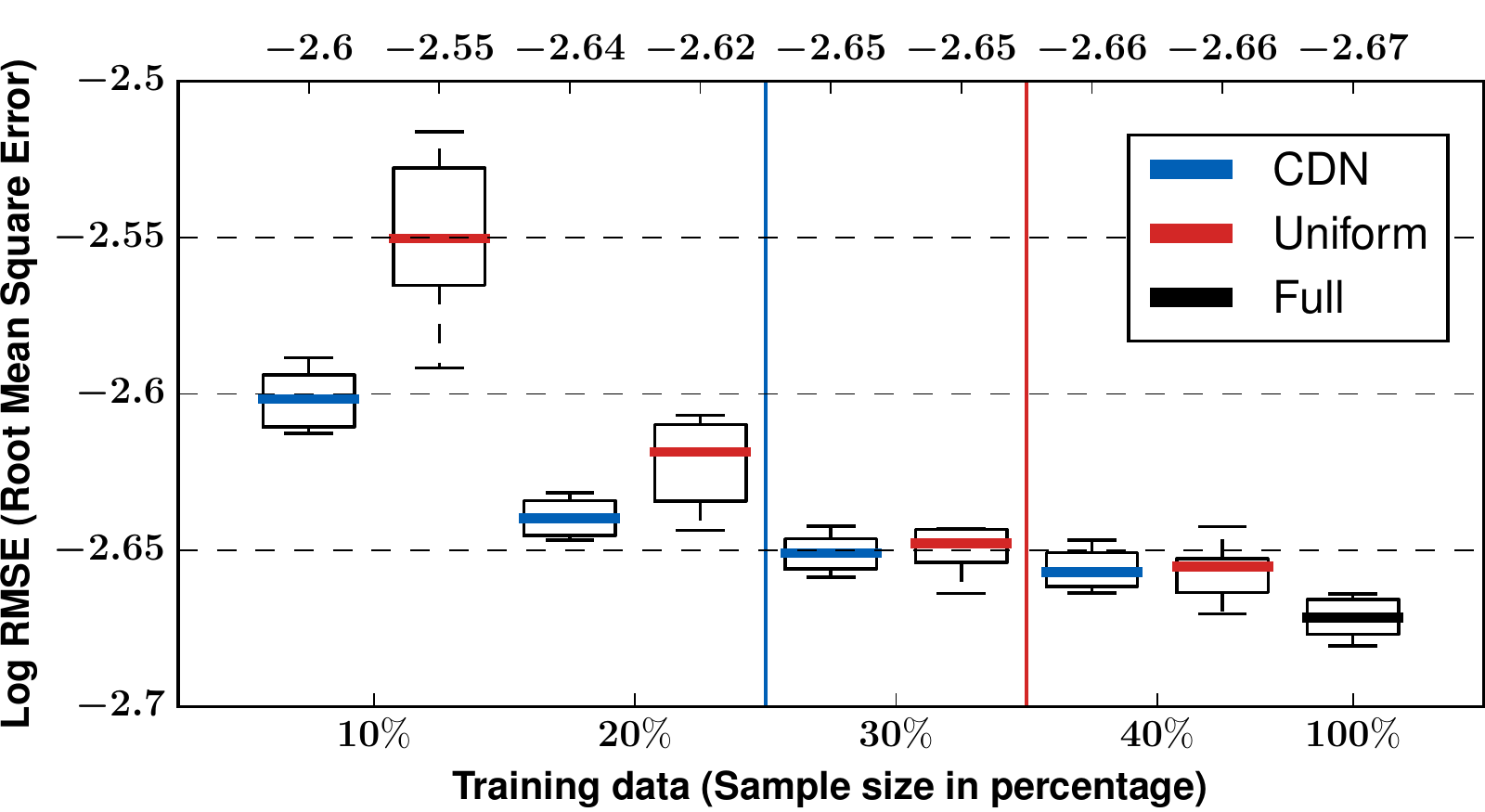}
&\includegraphics[width=1\linewidth]{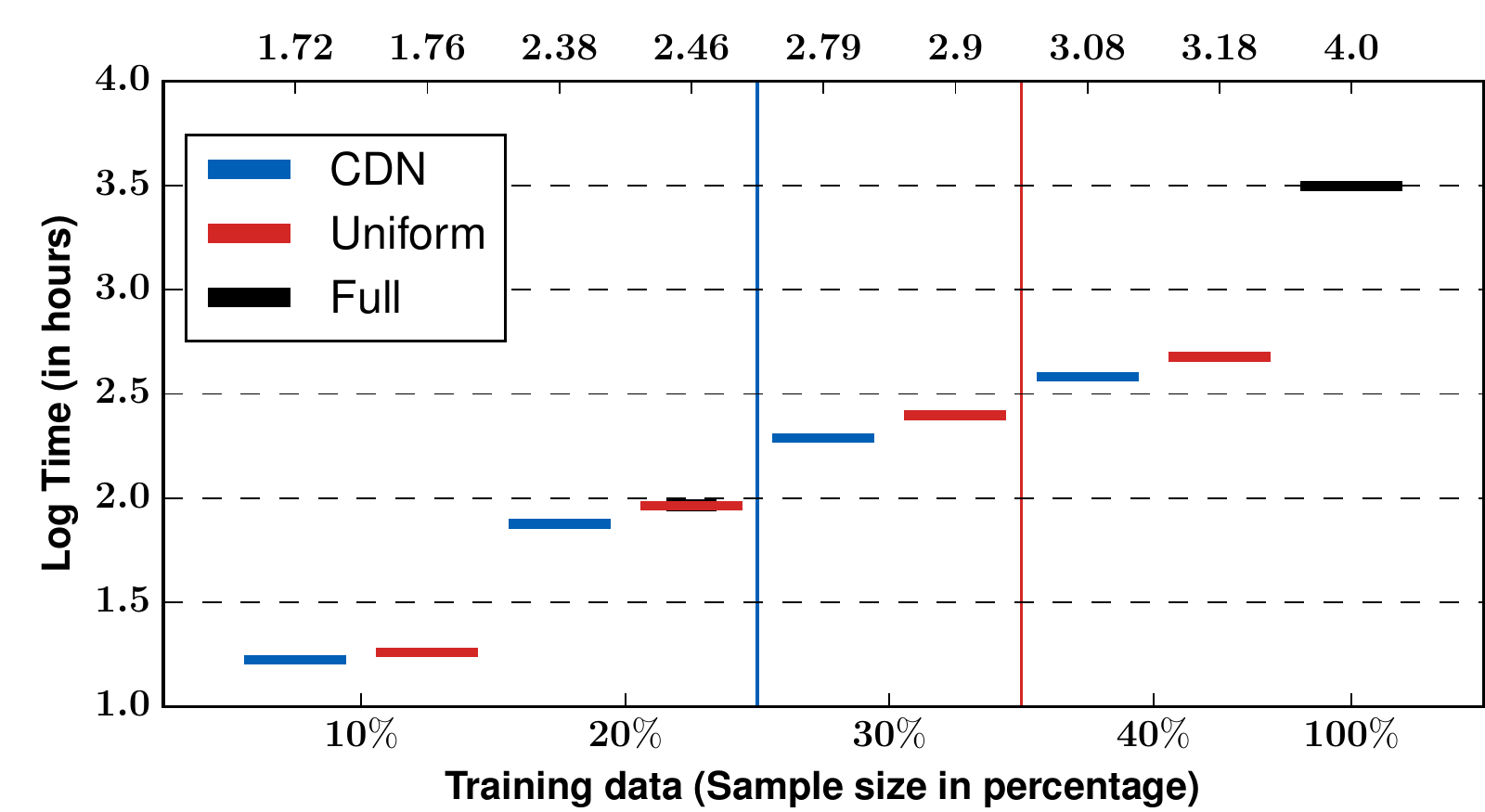}\\
\includegraphics[width=1\linewidth]{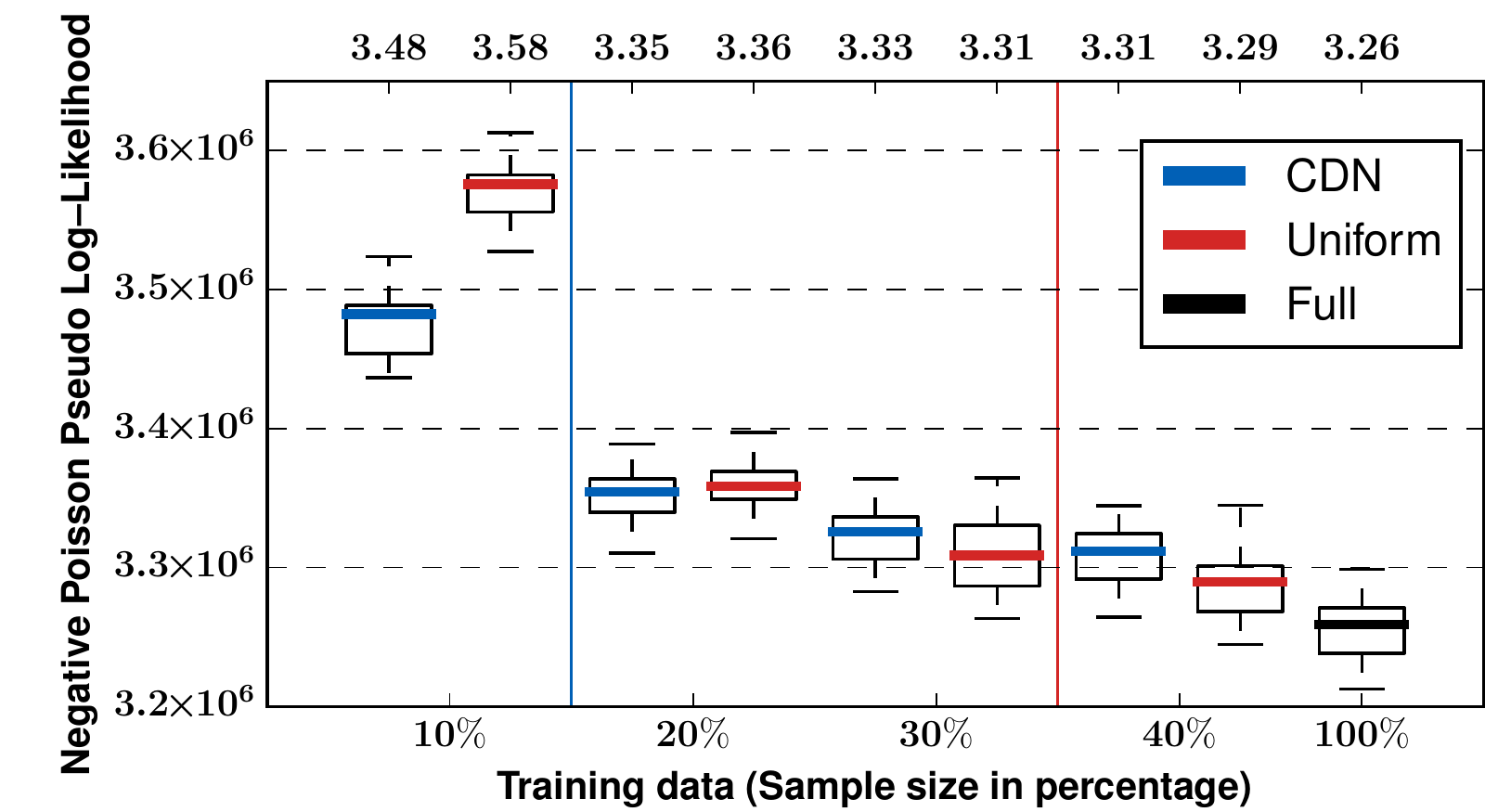}
&\includegraphics[width=1\linewidth]{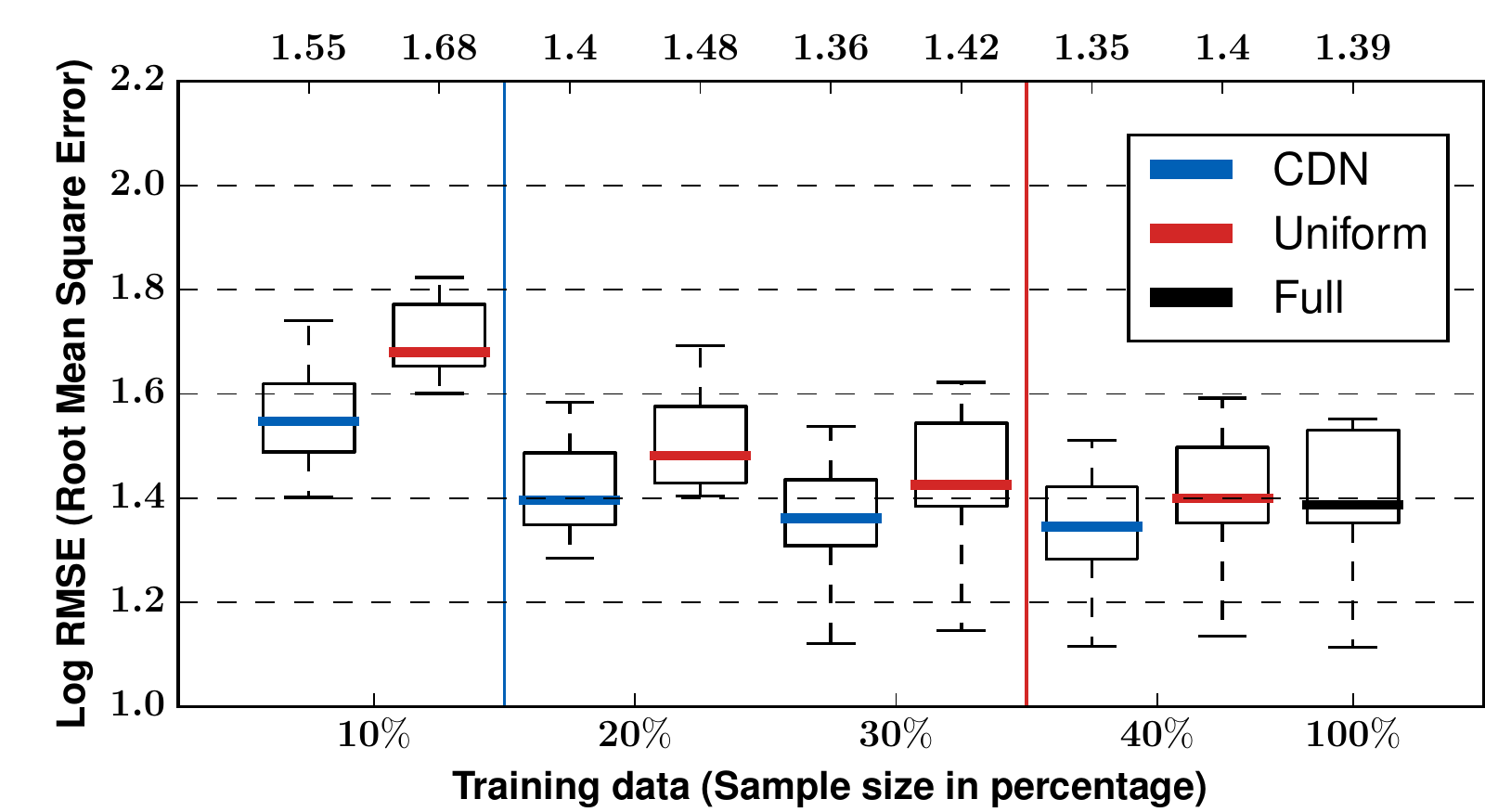}
&\includegraphics[width=1\linewidth]{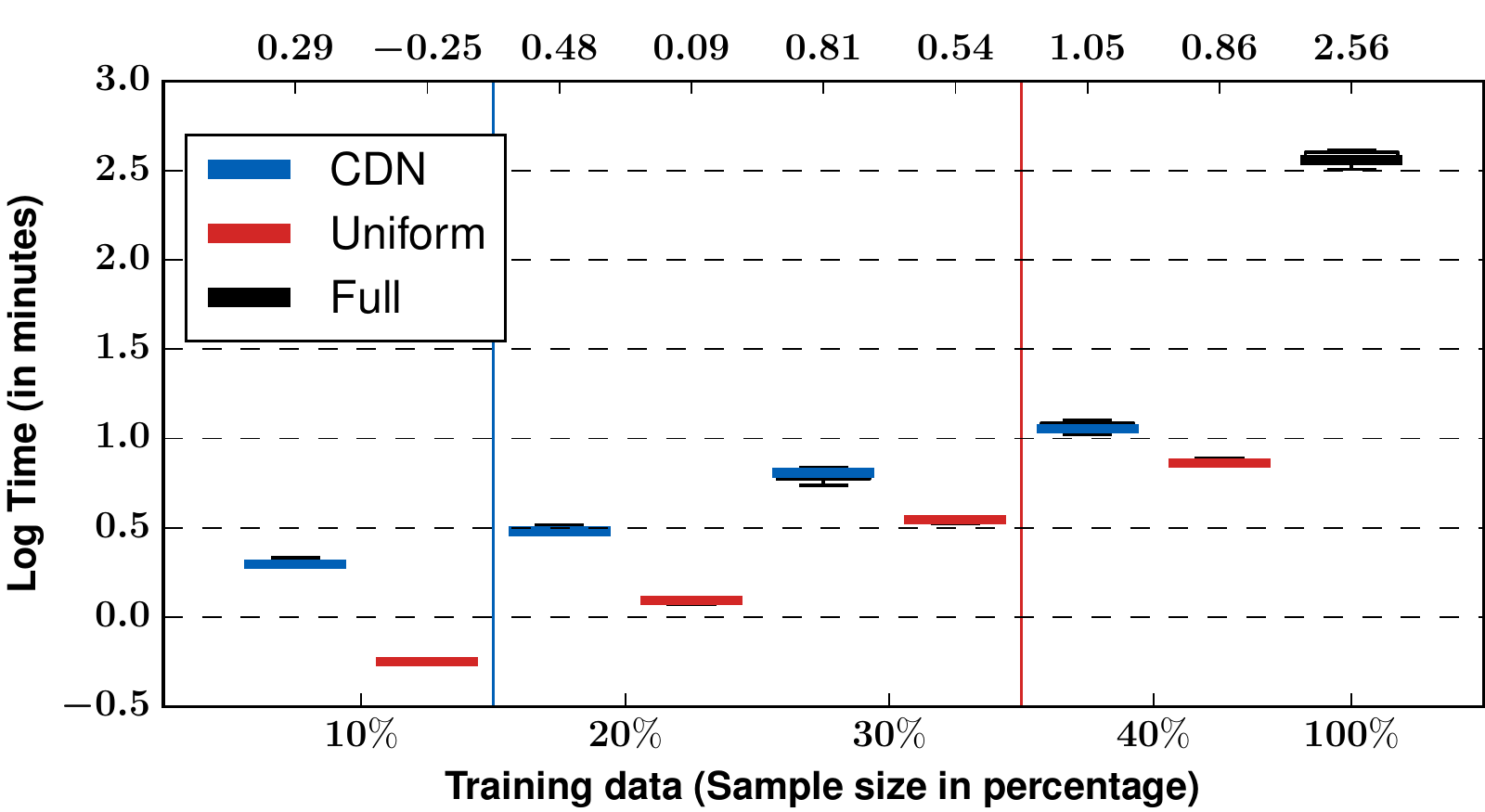}\\
{\small\bf Negative Log Pseudo Likelihood} & {\small\bf RMSE} & {\small\bf Training Time}
\end{tabular}
\caption{{\bf (Q1)} Performance (the lower, the better) of Gaussian CDNs on MNIST (upper row) and Poisson CNDs on the traffic dataset (lower row) 10-fold cross-validated. Shown are the negative log pseudo likelihood (left), the squared error loss (middle, in log-space) as well as the training time (right, in log-space) on the y-axis for different proportions of the data  sampled (x axis). Please note the jump in the x-axis after $40\%$.  As one can see, CDNs (blue) quickly approach the predictive performance of the full dataset (Full, black). Uniform sampling (Uniform, red) does not perform as well as CDNs.  Moreover, CDNs can be orders of magnitude faster than DNs on the full dataset and scale similar to uniform sampling. This is also supported by the vertical lines. They denote the mean performances (the more to the left, the better) on the top axes. (Best viewed in color)}
\label{fig:cdns_exp}
\end{figure*}

There is still a missing piece in our argumentation. In the previous section we have used that the coreset construction is an $\eps$-subspace embedding for the columnspace of the whole data set including the dependent variable, i.e., for $[X,\ln(\lambda)]$. We face two problems. First, $\lambda$ is only implicitly given in the data, but is not explicitly available. Second, $\lambda$ is a vector derived from $X^{\setminus i}$ in our setting and might be different for any of the $d$ instances. Fortunately, it was shown via more complicated arguments \cite{DrineasMM08}, that it is sufficient for a good approximation, if the sampling is done obliviously to the dependent variable. The intuition comes from the fact that the loss of any point in the subspace can be expressed via the projection of $\ln(\lambda)$ onto the subspace spanned by $X$, and the residual of its projection. A good approximation of the subspace implicitly approximates the projection of any fixed vector, which is then applied to the residual vector of the orthogonal projection. This solves the first problem, since it is only necessary to have a subspace embedding for $X$. The second issue can be addressed by increasing the sample size by a factor of $O(\log d)$ for boosting the error probability to $O(1/d)$ and taking a union bound.

\section{Empirical Illustration}
Our intention here is to corroborate our theoretical results by investigating empirically the following questions:
%To evaluate CDNs, we focus our attention on four fronts: learning time, prediction performance, structure recovery and error bounds. To that end we investigate the following questions: 
\textbf{(Q1)} How does the performance of CDNs compare to DNs with access to the full training data set and to a uniform sample from the training data set? and how does the empirical error behave according to the sample sizes?  \textbf{(Q2)} Do coresets affect the structure recovered by the DN? 
To this aim, we implemented (C)DNs in Python calling R. All experiments ran on a Linux machine (56 cores, 4 GPUs, and 512GB RAM).

%\subsection{Benchmarks on MNIST and Traffic Data \textbf{(Q1)}} 
\textbf{Benchmarks on MNIST and Traffic Data \textbf{(Q1)}:} 
We considered two datasets. In a first experiment, we used the MNIST\footnote{\url{http://yann.lecun.com/exdb/mnist/}}
data set of handwritten labeled digits. We employed the training set consisting of 55000 images, each with 784 pixels, for a total of 43,120,000 measurements, and trained Gaussian DNs on it.
The second data set we considered contains traffic count measurements on selected roads around the city of Cologne in Germany \cite{ide2015lte}. It consists of 7994 time-stamped measurements taken by 184 sensors for a total of 1,470,896 measurements. 
On this dataset we trained Poisson DNs.
For each dataset, we performed 10 fold cross-validation for training a full DN (Full) using all the data, leverage score sampling coresets (CDNs), and uniform samples (Uniform), for different sample sizes. We then compared the predictions made by all the DNs and the time taken to train them. 
For the predictions on the MNIST dataset, we clipped the predictions to the range [0,1] for all the DNs. For the Traffic dataset, we computed the predictions $\floor{x}$ of every measurement $x$ rounded to the largest integer less than or equal to $x$.

\begin{table}[t!]
\centering
\begin{tabular}{rrrrr}\small
Sample & \multicolumn{2}{c}{MNIST} & \multicolumn{2}{c}{Traffic} \\
\cmidrule(l{.4em}){2-3}\cmidrule(l{.4em}){4-5}
portion & GCDN & GUDN & PCDN & PUDN \\ \hline
10\% & \bf 18.03\% & 11162.01\% & \bf 6.81\% & 9.6\% \\
20\% & \bf 0.57\% & 13.86\% & \bf 2.9\% & 3.17\% \\
30\% & \bf 0.01\% & 13.33\% &  2.04\% & \bf 1.68\% \\
40\% & \bf 0.01\% & 2.3\% &  1.59\% & \bf 0.99\%
\end{tabular}
\caption{{\bf (Q1)} Comparison of the empirical relative error (the lower, the better). Best 
results per dataset are bold. Both Gaussian (GCDNs) and Poisson (PCDNs) CDNs recover the model well, with a fraction of the training data. Uniformly sampled DNs (UDNs) lag behind as the sample size drops.\label{tab:errors}}
\end{table}

Fig.~\ref{fig:cdns_exp} summarizes the results. As one can see, CDNs outperform DNs trained on full data and
are orders of magnitude faster. Compared to uniform sampling, coresets are competitive. 
Actually, as seen on the traffic dataset, CDNs can have more predictive power than the ``optimal'' model using the full data. This is in line with \citeauthor{Mahoney11}~(\citeyear{Mahoney11}), who observed that coresets implicitly introduce regularization and lead to more robust output.  
Table \ref{tab:errors} summarizes the empirical relative errors $|f(X,\tilde\gamma) - f(X,\gamma^*)|/f(X,\gamma^*)$ between (C/U)DNs $\tilde\gamma$ and DNs $\gamma^*$ trained on all the data. CDNs clearly recover the original model, at a fraction of training data. 
Overall, this answers {\bf (Q1)} affirmatively.

\begin{figure*}[t!]
\centering
\begin{tabular}{ccc}
\includegraphics[width=0.31\linewidth]{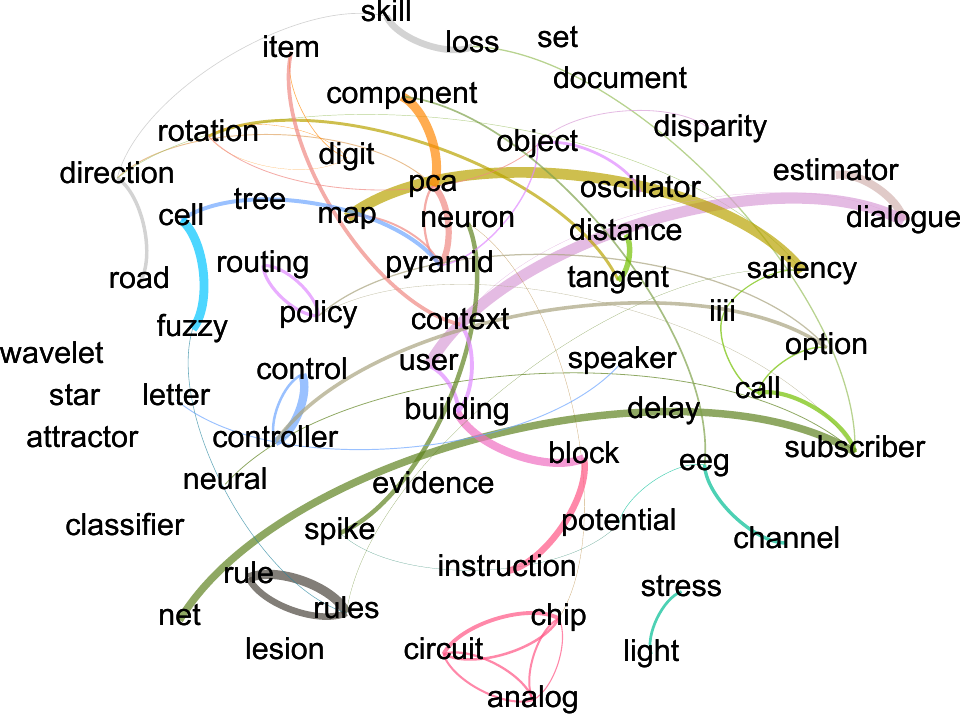}&
\includegraphics[width=0.31\linewidth]{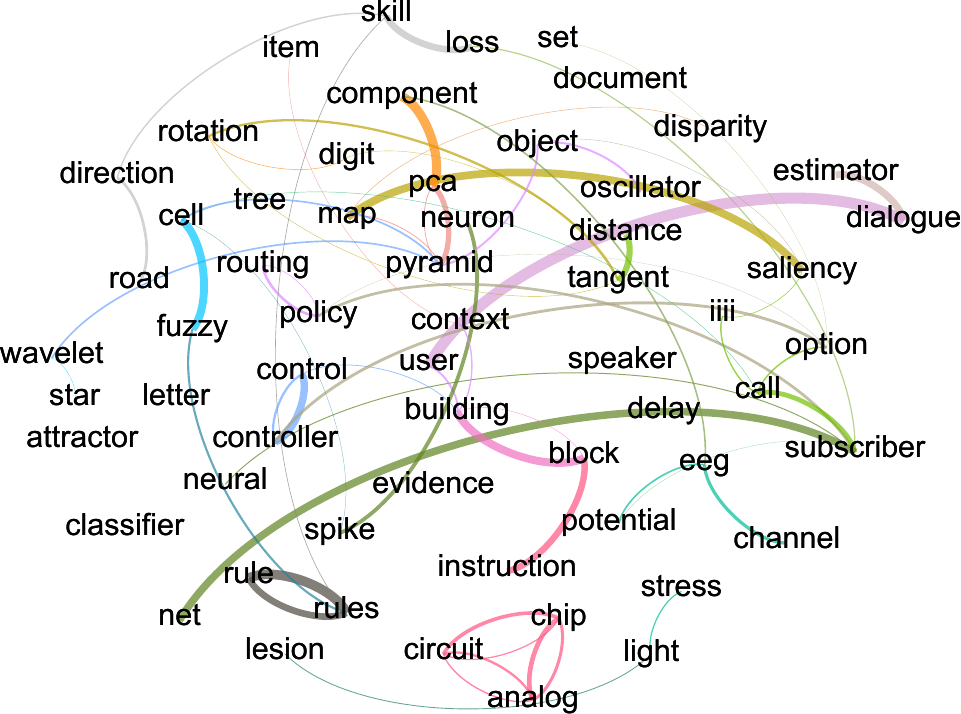}&
\includegraphics[width=0.31\linewidth]{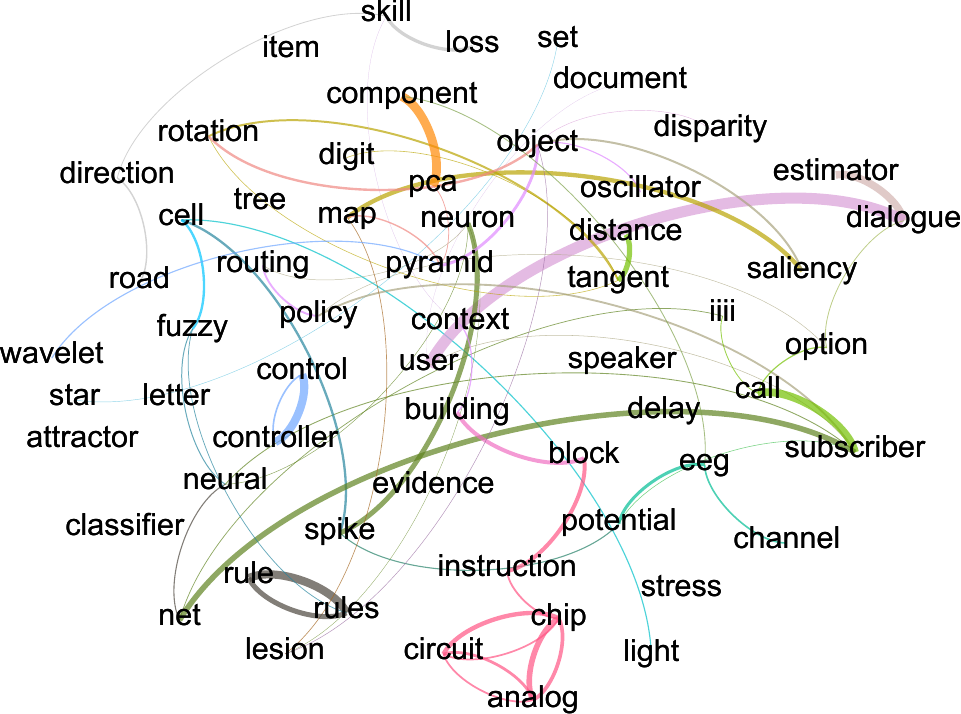} \\ 
\multicolumn{3}{r}{\small\bf Gaussian CDN}\\ \hline
\multicolumn{3}{r}{\small\bf Poisson CDN}\\
\includegraphics[width=0.31\linewidth]{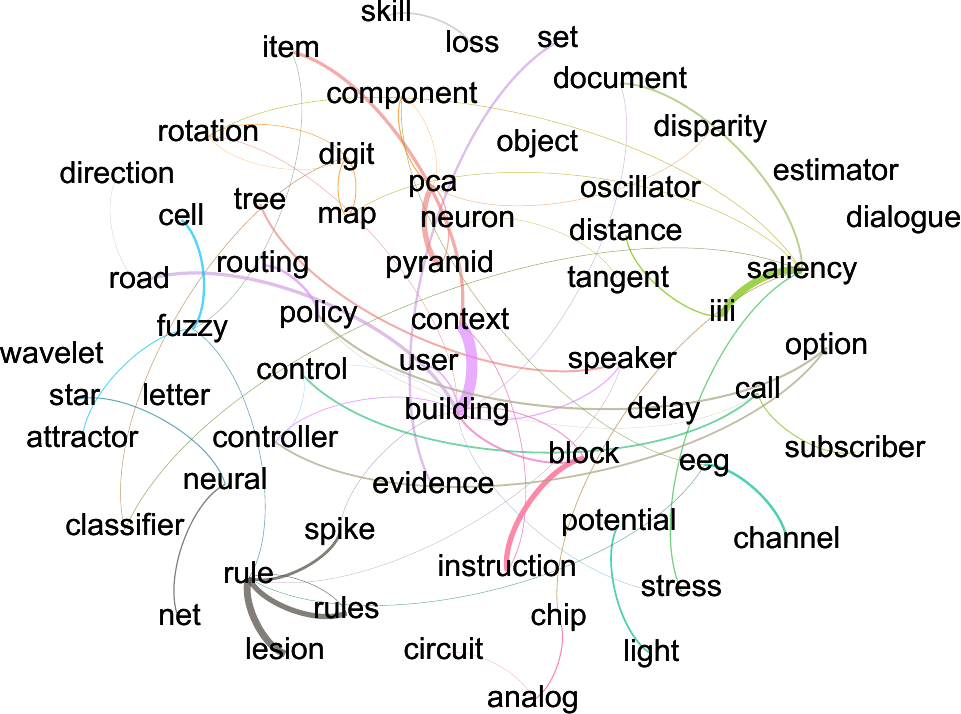}&
\includegraphics[width=0.31\linewidth]{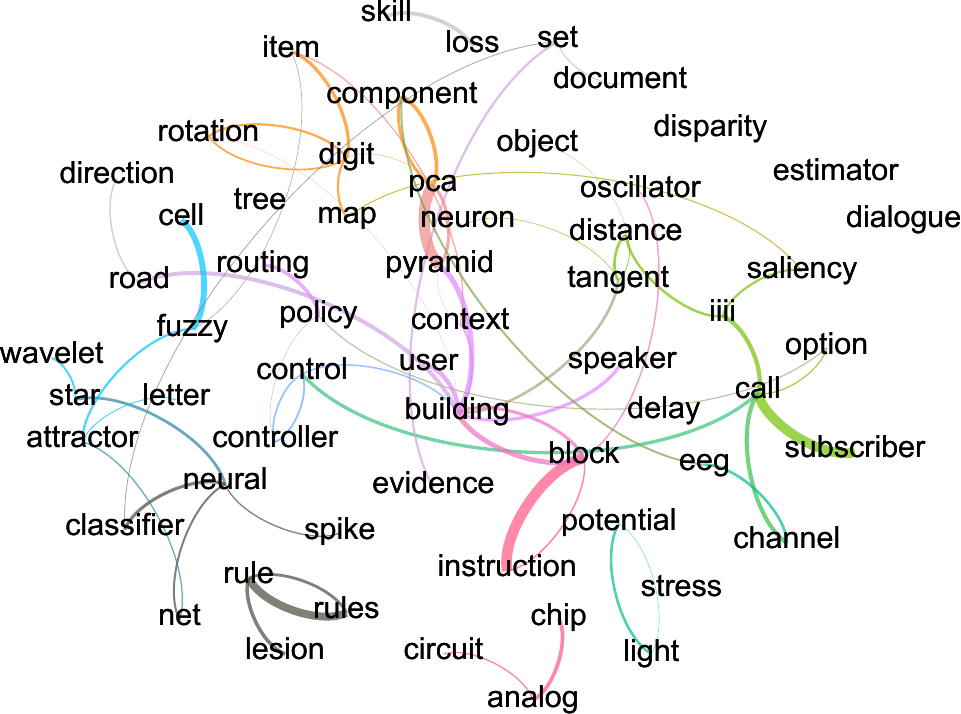}&
\includegraphics[width=0.31\linewidth]{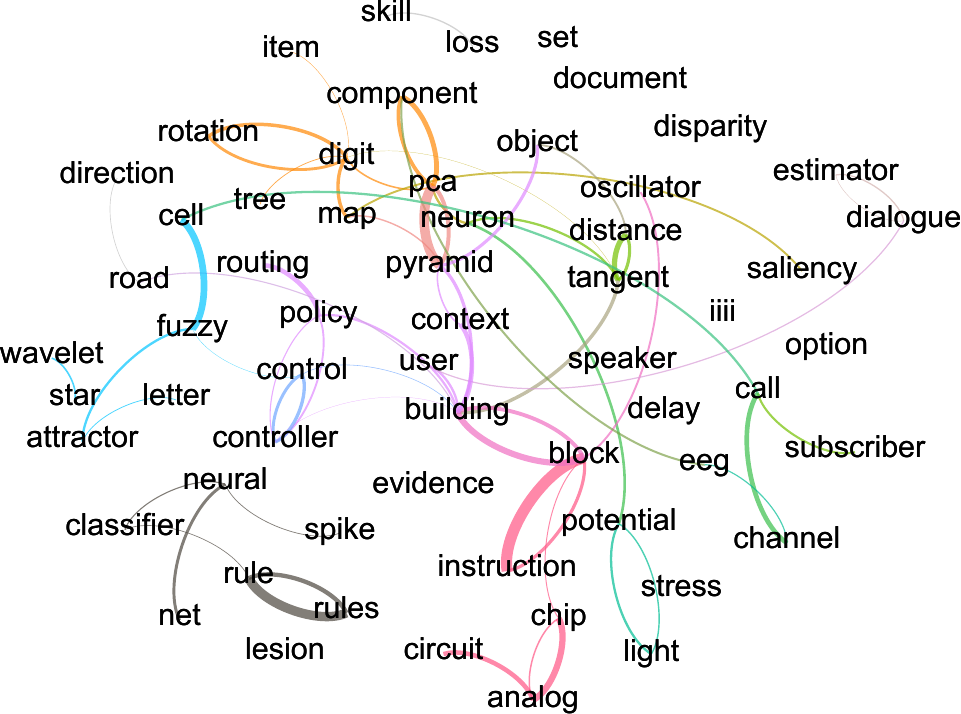}
\end{tabular}
\caption{{\bf (Q2)} Elucidating the relationships between random variables. Shown are the (positive) dependency structures of Gaussian (top) and Poisson (bottom) CDNs on NIPS and different learning sampling sizes: using 40\% (Left) , 70\% (Middle) and 100\% (Right). The edges show the 70 top thresholded positive coefficients of the GLMs. The colors of the edges represent modularity. As one can see, CDNs elucidate relationships among the words that make semantically sense and approach the structure learned using the full dataset. For a quantitative assessment, see Tab.~\ref{tab:FN}. (Best viewed in color)\label{fig:strct_exp}}
\end{figure*}

\textbf{Relationship Elucidation \textbf{(Q2)}:}
We investigated the performance of CDNs when recovering the graph structure of word interactions from a text corpus. For this purpose, we used the NIPS\footnote{\url{https://archive.ics.uci.edu/ml/datasets/bag+of+words}} bag-of-words dataset. It contains 1,500 documents with a vocabulary above 12k words. We considered the 100 most frequent words.

Fig.~\ref{fig:strct_exp} illustrates the results qualitatively. It shows three CDNs of sampling sizes 40\%, 70\% and 100\% for Gaussians (top) after a $\log(x+1)$ transformation and for Poissons (bottom): CDNs capture well the gist of the NIPS corpus. Table~\ref{tab:FN} confirms this quantitatively. It shows the Frobenius norms between the DNs: CDNs capture the gist better than naive, i.e., uniform sampling. This answers {\bf (Q2)} affirmatively.

To summarize our empirical results, the answers to questions {\bf (Q1)} and {\bf (Q2)} show the benefits of CDNs.

\section{Conclusions}
Inspired by the question of how  we can train  graphical  models  on  a  massive  dataset, we 
have studied coresets for estimating Dependency networks (DNs). We established the first rigorous
guarantees for obtaining compressed $\eps$-approximations of Gaussian DNs for
large data sets. We proved worst-case impossibility results on coresets for Poisson DNs. A review of log-normal Poisson modeling of counts provided deep insights into why our coreset construction still performs well for count data in practice. 

Our experimental results demonstrate, the resulting Core Dependency Networks (CDNs) can achieve significant gains over no or naive sub-sampling,
even in the case of count data, making it possible to learn models on much larger datasets using the same hardware. 
\begin{table}[t!]

\centering
\begin{tabular}{ccccc}\small
Sample & \multicolumn{2}{c}{UDN} & \multicolumn{2}{c}{CDN}\\
\cmidrule(l{.4em}){2-3}\cmidrule(l{.4em}){4-5}
portion & Gaussian & Poisson & Gaussian & Poisson \\ \hline
40\%  & 9.0676 & 6.4042&  \textbf{3.9135}  & \textbf{0.6497} \\ 
70\% & 4.8487  & 1.6262  & \textbf{2.6327} & \textbf{0.3821} 
\end{tabular}
\caption{{\bf (Q2)} Frobenius norm of the difference of the adjacency matrices (the lower, the better) recovered by DNs trained on the full data and trained on a uniform subsample (UDN) resp. coresets (CDNs) of the training data. The best results per statiscal type (Gaussian/Poisson) are bold. CDNs recover the structure better than UDNs.\label{tab:FN}}
\end{table}

CDNs provide several interesting avenues for future work. The conditional independence assumption opens the door to explore hybrid multivariate models, 
where each variable can potentially come from a different GLM family or link function, on massive data sets. This can further be used to hint at independencies among variables in the multivariate setting, making them useful in many other large data applications. Generally, our results may pave the way to establish coresets for deep models using the close connection between dependency networks and deep generative stochastic networks \cite{bengioLAY14}, sum-product networks \cite{Poon2011,Molina2017}, as well as other statistical models that build multivariate distributions from univariate ones \cite{yang15}.

{\bf Acknowledgements:} This work has been supported by Deutsche Forschungsgemeinschaft
(DFG) within the Collaborative Research Center SFB 876 "Providing Information by Resource-Constrained Analysis", projects B4 and C4.

\newpage
\bibliography{ref,bibliography}

\end{document}